\documentclass[reprint,amsmath,amssymb,aps,pra,groupedaddress,floatfix]{revtex4-2}
\usepackage{tikz}
\usepackage{pgfplots}
\usepackage{amsthm}
\usepackage{algorithmic}
\usepackage{graphicx}
\usepackage{hyperref}
\usepackage{booktabs}
\usepackage{xcolor}
\usepackage{bm}
\usepackage{mathtools}
\usepackage{microtype}
\usepackage{float} 
\usetikzlibrary{
  arrows.meta,
  shapes.geometric,
  positioning,
  calc,
  decorations.pathreplacing,
  backgrounds,
  fit,
  patterns,
  3d,
  plotmarks
}

\pgfplotsset{compat=1.18}

\definecolor{set2blue}{RGB}{102,194,165}
\definecolor{set2orange}{RGB}{252,141,98}
\definecolor{set2green}{RGB}{141,160,203}
\definecolor{set2pink}{RGB}{231,138,195}
\definecolor{set2lime}{RGB}{166,216,84}
\definecolor{set2yellow}{RGB}{255,217,47}
\definecolor{set2tan}{RGB}{229,196,148}
\definecolor{set2grey}{RGB}{179,179,179}

\definecolor{rsgncol}{RGB}{70,130,180}
\definecolor{ssacol}{RGB}{230,120,50}
\definecolor{couplecol}{RGB}{60,160,80}
\definecolor{fastcol}{RGB}{200,50,50}
\definecolor{slowcol}{RGB}{50,80,180}

\newtheorem{theorem}{Theorem}
\newtheorem{lemma}[theorem]{Lemma}
\newtheorem{proposition}[theorem]{Proposition}

\newtheorem{assumption}{Assumption}
\newtheorem{remark}{Remark}

\makeatletter
\newcommand{\inlinetablecaption}[1]{%
  \refstepcounter{table}%
  \vspace{\abovecaptionskip}%
  \@makecaption{\tablename~\thetable}{#1}%
  \vspace{\belowcaptionskip}%
}
\newcommand{\inlinefigurecaption}[1]{%
  \refstepcounter{figure}%
  \vspace{\abovecaptionskip}%
  \@makecaption{\figurename~\thefigure}{#1}%
  \vspace{\belowcaptionskip}%
}
\makeatother

\newcommand{\R}{\mathbb{R}}

\newcommand{\Ind}{\mathbb{1}}
\newcommand{\norm}[1]{\left\|#1\right\|}

\newcommand{\Bball}{\mathbb{B}^d}
\newcommand{\dH}{d_{\mathcal{H}}}

\begin{document}

\title{Hebbian-Oscillatory Co-Learning}

\author{Hasi Hays}
\email{hasih@uark.edu}
\affiliation{Department of Chemical Engineering, University of Arkansas, Fayetteville, AR 72701, USA}
\date{\today}

\begin{abstract}
We introduce Hebbian-Oscillatory Co-Learning (HOC-L), a unified two-timescale dynamical framework for joint structural plasticity and phase synchronization in bio-inspired sparse neural architectures. HOC-L couples our two recent frameworks: the hyperbolic sparse geometry of Resonant Sparse Geometry Networks (RSGN), which employs Poincar\'{e} ball embeddings with Hebbian-driven dynamic sparsity, and the oscillator-based attention of Selective Synchronization Attention (SSA), which replaces dot-product attention with Kuramoto-type phase-locking dynamics. The key mechanism is synchronization-gated plasticity: the macroscopic order parameter $r(t)$ of the oscillator ensemble gates Hebbian structural updates, so that connectivity consolidation occurs only when sufficient phase coherence signals a meaningful computational pattern. We prove convergence of the joint system to a stable equilibrium via a composite Lyapunov function and derive explicit timescale separation bounds. The resulting architecture achieves $O(n \cdot k)$ complexity with $k \ll n$, preserving the sparsity of both parent frameworks. Numerical simulations confirm the theoretical predictions, demonstrating emergent cluster-aligned connectivity and monotonic Lyapunov decrease.
\end{abstract}


\maketitle

\section{Introduction}
\label{sec:introduction}

The capacity of biological neural systems to learn, adapt, and compute efficiently arises from a sophisticated interplay between multiple mechanisms operating on different temporal scales \cite{buzsaki2004}. At the slow timescale of hours to days, synaptic connections strengthen or weaken according to Hebbian plasticity rules: neurons that fire together wire together \cite{hebb1949}. At the fast timescale of milliseconds, oscillatory synchronization among neural populations coordinates information flow and gates communication between brain regions \cite{fries2005, singer1999}. These two processes are not independent; rather, accumulating neuroscientific evidence indicates that oscillatory coherence directly modulates the conditions under which structural plasticity occurs \cite{fell2011, markram2011}. Phase-locked neural assemblies preferentially strengthen their mutual synaptic connections, while desynchronized populations undergo synaptic weakening. Despite this well-established biological principle, contemporary artificial neural network architectures treat structural learning and dynamic coordination as separate computational concerns. Standard deep learning \cite{hays2026encyclopedia} optimizes a fixed network topology through gradient descent, with no mechanism for activity-dependent structural rewiring. Sparse network methods such as the lottery ticket hypothesis \cite{frankle2019} and adaptive sparse training \cite{mocanu2018} address the structural dimension but lack any oscillatory coordination mechanism. Conversely, oscillator-based neural computation models \cite{kuramoto1984, hoppensteadt1997} capture temporal coordination but typically operate on fixed connectivity graphs without plasticity.

In our recent work, we have developed two complementary frameworks that each address one side of this gap. Resonant Sparse Geometry Networks (RSGN) \cite{rsgn2025} introduced a brain-inspired architecture that embeds neural representations in the Poincar\'{e} ball model of hyperbolic space, enabling naturally hierarchical and sparse connectivity patterns. Crucially, RSGN employs a two-timescale learning scheme in which fast gradient-based optimization handles immediate task objectives, while slow Hebbian updates reshape the structural connectivity of the sparse graph. This input-dependent dynamic sparsity yields architectures where the active computational subnetwork adapts to each input, achieving substantial efficiency gains without sacrificing expressiveness.

In a companion study, we proposed Selective Synchronization Attention (SSA) \cite{ssa2026}, which replaces the standard dot-product self-attention mechanism \cite{hays2026attention} in transformers with a Kuramoto oscillator model. In SSA, each token is assigned a learnable intrinsic frequency, and attention weights emerge from the degree of phase-locking between token oscillators rather than from inner products in key-query space. The macroscopic order parameter $r(t)$ of the oscillator ensemble provides a natural, differentiable sparsity mechanism: only sufficiently synchronized token pairs receive non-negligible attention weight, while desynchronized pairs are effectively masked. Together, these two frameworks address complementary aspects of the same fundamental biological phenomenon. RSGN captures the structural plasticity that operates on the slow timescale, namely how connectivity patterns emerge from correlated activity, while SSA captures the oscillatory dynamics that operate on the fast timescale, namely how phase coherence selects which information flows through existing connections. Their unification is not merely an engineering convenience but a desirable property for any model that aspires to capture the full richness of neural computation.

We introduce \textbf{Hebbian-Oscillatory Co-Learning (HOC-L)}, a unified two-timescale dynamical system that formally couples structural plasticity with phase synchronization. The core insight of HOC-L is that the order parameter $r(t)$ from SSA's oscillator dynamics serves as a natural gate for RSGN's Hebbian structural updates. When a sufficient subset of oscillators achieves phase coherence ($r(t) > r_c$ for a critical threshold $r_c$), this signals that the network has identified a meaningful computational pattern, and the corresponding structural connections are strengthened via a Hebbian rule. This creates a virtuous cycle: synchronization identifies important pathways, Hebbian plasticity consolidates them structurally, and the resulting structural changes facilitate future synchronization.

The principal contributions of this paper are as follows:
\begin{enumerate}
    \item We formulate the first unified dynamical system coupling Kuramoto-type phase synchronization with synchronization-gated Hebbian structural plasticity, grounded in the frameworks of RSGN \cite{rsgn2025} and SSA \cite{ssa2026}.
    \item We prove convergence of the joint two-timescale system to a stable equilibrium using a composite Lyapunov function, establishing explicit conditions on the timescale separation ratio $\tau_{\text{fast}} / \tau_{\text{slow}}$.
    \item We derive the complete HOC-L architecture with $O(n \cdot k)$ computational complexity where $k \ll n$, preserving the sparsity benefits of both parent frameworks.
    \item We propose comprehensive experimental protocols for validating the theoretical predictions across multiple benchmark domains.
\end{enumerate}

We develop the complete theoretical framework, present convergence proofs and stability analysis, describe the concrete architecture with supporting numerical simulations, and outline experimental protocols for empirical validation.

\section{Related Work}
\label{sec:related}

HOC-L draws upon and extends several established research directions spanning neuroscience-inspired learning rules, oscillatory computation, geometric representations, and sparse architectures. The principle that correlated neural activity drives synaptic strengthening was first articulated by Hebb \cite{hebb1949} and subsequently refined through the discovery of spike-timing-dependent plasticity \cite{bi2001, caporale2008}. In the machine learning context, Hebbian principles have informed unsupervised feature learning, competitive learning, and self-organizing maps \cite{hebb1949, bi2001, markram2011}. In our recent work, RSGN \cite{rsgn2025}, we demonstrated that Hebbian updates can be integrated into a gradient-trained sparse architecture operating in hyperbolic space, enabling input-dependent structural adaptation at a slow timescale while fast gradient updates handle task-specific optimization. This two-timescale design principle provides one of the two foundational pillars of the present work.

The second pillar originates from oscillatory neural computation. The Kuramoto model \cite{kuramoto1984} provides the canonical mathematical framework for synchronization in coupled oscillator populations, and its application to neuroscience has elucidated how gamma-band oscillations coordinate neural communication \cite{fries2005, buzsaki2004} and how phase coherence supports memory formation \cite{fell2011}. More broadly, the integration of continuous dynamical systems into neural architectures has been advanced by Neural ODEs \cite{chen2018neural}, which parameterize hidden state evolution as ordinary differential equations; HOC-L shares the continuous-time perspective but focuses specifically on oscillatory phase dynamics rather than general-purpose ODE solvers. The theoretical analysis of Kuramoto dynamics, including the characterization of phase transitions via the macroscopic order parameter, has been extensively developed \cite{strogatz2000, acebron2005, breakspear2010}. In our companion work, SSA \cite{ssa2026}, we brought these ideas into the transformer architecture by replacing dot-product attention with Kuramoto-type synchronization weights, using the order parameter $r$ as a coherence measure and a learned compatibility kernel $C(\omega_i, \omega_j)$ to modulate coupling strength. The present work extends SSA's oscillatory dynamics by coupling them to structural plasticity through a synchronization gate.

The geometric substrate for HOC-L's structural component builds on advances in graph neural networks and hyperbolic representations. Graph neural networks \cite{scarselli2009, kipf2017, velickovic2018, hays2026raggnn} operate on relational structures but typically assume fixed graph topology. Recent work has begun to address this limitation: differentiable graph structure learning methods \cite{franceschi2019learning} and neural relational inference \cite{kipf2018neural} learn graph topology jointly with node representations, while dynamic graph neural networks such as EvolveGCN \cite{pareja2020evolvegcn} adapt graph convolution parameters over time. HOC-L differs from these approaches by coupling topology evolution to an explicit oscillatory coherence signal rather than learning structure purely from task gradients. Hyperbolic embeddings \cite{nickel2017, ganea2018, chami2019, hays2025hierarchical} offer natural representations for hierarchical data by exploiting the exponential volume growth of hyperbolic space. In RSGN \cite{rsgn2025}, we unified these ideas by embedding sparse neural graph structures in the Poincar\'{e} ball, where geodesic distances determine connection strength and the graph topology evolves through Hebbian plasticity. HOC-L inherits this hyperbolic structural backbone while augmenting it with oscillatory coordination. The formal analysis of HOC-L relies on the mathematical theory of two-timescale stochastic approximation \cite{borkar1997, borkar2008}, which provides convergence guarantees for systems where coupled dynamical processes evolve at different rates. In such systems, the fast process equilibrates quasi-statically with respect to the slowly varying parameter, enabling separate analysis of each timescale. This framework, which has found applications in reinforcement learning and game theory, accommodates HOC-L's coupling of fast oscillatory dynamics with slow Hebbian plasticity.

Finally, HOC-L connects to the broader literature on sparse neural architectures. The lottery ticket hypothesis \cite{frankle2019} established that dense networks contain sparse subnetworks capable of matching full network performance when trained in isolation, adaptive sparse training methods \cite{mocanu2018} dynamically evolve network topology through periodic pruning and regrowth, and biologically inspired deep rewiring \cite{bellec2018deep} maintains constant connectivity during training by stochastically adding and removing connections. In RSGN \cite{rsgn2025}, we advanced this line by introducing input-dependent sparsity driven by hyperbolic geometry, where each input activates a different sparse subgraph determined by geodesic neighborhoods. HOC-L preserves this dynamic sparsity while adding an oscillatory coordination layer that further refines which connections are computationally active at any given moment.

\section{Theoretical Framework}
\label{sec:framework}

This section establishes the mathematical foundations upon which HOC-L is constructed, drawing from hyperbolic geometry, oscillator theory, and Hebbian plasticity.

\subsection{Hyperbolic Geometry and the Poincar\'{e} Ball}

Following our RSGN framework \cite{rsgn2025} and the foundational work on Poincar\'{e} embeddings \cite{nickel2017}, we operate in the Poincar\'{e} ball model $\Bball = \{\mathbf{x} \in \R^d : \norm{\mathbf{x}} < 1\}$ equipped with the Riemannian metric tensor $g_{\mathbf{x}} = \lambda_{\mathbf{x}}^2 g_E$ where $\lambda_{\mathbf{x}} = 2/(1 - \norm{\mathbf{x}}^2)$ is the conformal factor and $g_E$ is the Euclidean metric. The geodesic distance between two points $\mathbf{x}, \mathbf{y} \in \Bball$ is given by:
\begin{equation}
\label{eq:hyperbolic_distance}
\dH(\mathbf{x}, \mathbf{y}) = \operatorname{arcosh}\!\left(1 + \frac{2\norm{\mathbf{x} - \mathbf{y}}^2}{(1 - \norm{\mathbf{x}}^2)(1 - \norm{\mathbf{y}}^2)}\right).
\end{equation}

This distance metric is central to RSGN's sparse graph construction: nodes $i$ and $j$ are connected in the sparse computational graph if and only if $\dH(\mathbf{x}_i, \mathbf{x}_j) < \delta$ for a learned threshold $\delta$, yielding an input-dependent sparse neighborhood of size $k \ll n$ for each node \cite{rsgn2025}.

\subsection{Kuramoto Oscillator Dynamics}

Following our SSA framework \cite{ssa2026} and the classical Kuramoto framework \cite{kuramoto1984}, we associate each computational unit $i \in \{1, \ldots, N\}$ with an oscillator having phase $\theta_i(t) \in [0, 2\pi)$ and intrinsic frequency $\omega_i$. The coupled dynamics governing phase evolution are:
\begin{equation}
\label{eq:kuramoto}
\frac{d\theta_i}{dt} = \omega_i + K \cdot r \sum_{j=1}^{N} C(\omega_i, \omega_j) \sin(\theta_j - \theta_i),
\end{equation}
where $K > 0$ is the global coupling strength, $C(\omega_i, \omega_j)$ is a learned compatibility kernel that modulates pairwise coupling based on intrinsic frequencies, and $r(t)$ is the macroscopic order parameter defined as below. Following SSA \cite{ssa2026}, the compatibility kernel takes the form $C(\omega_i, \omega_j) = \exp\!\bigl(-\|\omega_i - \omega_j\|^2 / (2\sigma_C^2)\bigr)$ where $\sigma_C > 0$ is a bandwidth parameter; in the uniform-frequency limit ($\omega_i = \omega_j$) this reduces to $C = 1$, recovering the classical Kuramoto model. The macroscopic order parameter is:
\begin{equation}
\label{eq:order_parameter}
r(t) e^{i\psi(t)} = \frac{1}{N} \sum_{j=1}^{N} e^{i\theta_j(t)},
\end{equation}
where $\psi(t)$ is the mean phase. The magnitude $r \in [0, 1]$ quantifies the degree of global phase coherence: $r \approx 0$ indicates incoherence while $r \approx 1$ indicates full synchronization.

As we established in SSA \cite{ssa2026}, the synchronization-based attention weight between units $i$ and $j$ is:
\begin{equation}
\label{eq:ssa_weight}
A_{ij} = \max\!\big(0,\; K \cdot r \cdot C(\omega_i, \omega_j) - |\omega_i - \omega_j|\big).
\end{equation}

This formulation replaces the softmax-normalized dot-product attention of standard transformers \cite{vaswani2017} with a mechanism grounded in oscillatory coherence, providing inherent sparsity since $A_{ij} = 0$ whenever the frequency mismatch exceeds the coupling-weighted coherence.

\subsection{Hebbian Plasticity with Synchronization Gating}

Classical Hebbian learning \cite{hebb1949} prescribes that the synaptic weight $W_{ij}$ between neurons $i$ and $j$ increases when both neurons are simultaneously active. We augment this rule with a synchronization gate inspired by the biological observation that phase coherence modulates plasticity \cite{fell2011}:
\begin{equation}
\label{eq:hebbian_gated}
\Delta W_{ij} = \eta \cdot x_i \cdot x_j \cdot \Ind[r(t) > r_c],
\end{equation}
where $\eta > 0$ is the Hebbian learning rate, $x_i$ and $x_j$ are the activations of units $i$ and $j$, $\Ind[\cdot]$ is the indicator function, and $r_c \in (0, 1)$ is a critical synchronization threshold. This gate ensures that structural consolidation occurs only when the oscillator ensemble has achieved sufficient coherence, preventing spurious plasticity during transient or incoherent states.

\subsection{Two-Timescale Separation}

The mathematical foundation for analyzing coupled systems evolving at different rates is provided by two-timescale stochastic approximation theory \cite{borkar1997, borkar2008}. We define two characteristic timescales:
\begin{itemize}
    \item \textbf{Fast timescale} $\tau_{\text{fast}}$: The rate at which oscillator phases $\theta_i(t)$ evolve toward synchronization and gradient-based parameter updates occur.
    \item \textbf{Slow timescale} $\tau_{\text{slow}}$: The rate at which Hebbian structural updates modify the weight matrix $\mathbf{W}$.
\end{itemize}

The separation condition $\tau_{\text{fast}} \ll \tau_{\text{slow}}$ ensures that the oscillatory dynamics reach a quasi-stationary distribution before each Hebbian update fires, a condition we formalize in Section~\ref{sec:analysis}. For the theoretical convergence analysis in Section~\ref{sec:analysis}, we additionally require that the intrinsic frequencies $\omega_i$ are centered (zero mean) with narrowly distributed spread relative to the coupling strength. This regularity condition, stated formally as Assumption~\ref{ass:frequency}, enables the analysis to proceed in a co-rotating reference frame.

\section{The Hebbian-Oscillatory Co-Learning Model}
\label{sec:model}

We now present the complete HOC-L dynamical system, which couples the structural plasticity framework of our RSGN \cite{rsgn2025} with the oscillatory attention mechanism of our SSA \cite{ssa2026} through synchronization-gated Hebbian learning.

\subsection{Joint Dynamical System}

The HOC-L system is defined by two coupled differential equations operating at separated timescales:

\textbf{Fast dynamics (phase synchronization and gradient descent):}
\begin{equation}
\label{eq:fast}
\frac{d\theta_i}{dt} = \omega_i + K \cdot r \sum_{j=1}^{N} C(\omega_i, \omega_j) \sin(\theta_j - \theta_i),
\end{equation}

\textbf{Slow dynamics (Hebbian structural plasticity):}
\begin{equation}
\label{eq:slow}
\frac{dW_{ij}}{dt} = -\gamma W_{ij} + \eta \, x_i \, x_j \cdot \Ind[r(t) > r_c],
\end{equation}
where $\gamma > 0$ is a weight decay coefficient that prevents unbounded growth, and the activations $x_i$ are computed through the RSGN hyperbolic sparse graph using the current weight matrix $\mathbf{W}$, and SSA attention weights $A_{ij}$.

The coupling between timescales is bidirectional:
\begin{enumerate}
    \item \textbf{Synchronization gates plasticity:} The indicator $\Ind[r(t) > r_c]$ in Eq.~\eqref{eq:slow} ensures that Hebbian updates only fire when the fast oscillatory process has achieved sufficient coherence.
    \item \textbf{Structure modulates synchronization:} The weight matrix $\mathbf{W}$ determines the sparse graph topology over which both activations $x_i$ propagate, and oscillatory coupling occurs, thereby influencing the set of oscillators that can synchronize.
\end{enumerate}

\subsection{Synchronization-Gated Plasticity Mechanism}

The plasticity gate constitutes the central novel mechanism of HOC-L. We define the gating function more precisely as a smooth approximation to the indicator:
\begin{equation}
\label{eq:gate}
G(r) = \sigma\!\big(\beta(r - r_c)\big) = \frac{1}{1 + e^{-\beta(r - r_c)}},
\end{equation}
where $\beta > 0$ controls the sharpness of the gate and $\sigma$ denotes the logistic sigmoid. For $\beta \to \infty$, this recovers the hard indicator $\Ind[r > r_c]$, while finite $\beta$ provides a differentiable relaxation amenable to gradient-based optimization. The smoothed slow dynamics become:
\begin{equation}
\label{eq:slow_smooth}
\frac{dW_{ij}}{dt} = -\gamma W_{ij} + \eta \, x_i \, x_j \cdot G\!\big(r(t)\big).
\end{equation}

\subsection{Integration with RSGN Hyperbolic Structure}

HOC-L inherits the Poincar\'{e} ball embedding structure from RSGN \cite{rsgn2025}. Each unit $i$ maintains a hyperbolic embedding $\mathbf{z}_i \in \Bball$ that determines its position in the sparse computational graph. The sparse neighborhood of unit $i$ is defined as:
\begin{equation}
\label{eq:neighborhood}
\mathcal{N}_i = \{j : \dH(\mathbf{z}_i, \mathbf{z}_j) < \delta\},
\end{equation}
where $\delta$ is a learned or fixed distance threshold. The hyperbolic embeddings evolve via Riemannian gradient descent on the Poincar\'{e} ball, using the exponential map for parameter updates as described in RSGN \cite{rsgn2025} and \cite{nickel2017, ganea2018}:
\begin{equation}
\label{eq:riemannian_update}
\mathbf{z}_i \leftarrow \exp_{\mathbf{z}_i}\!\left(-\alpha_z \cdot \operatorname{grad}_{\mathbf{z}_i} \mathcal{L}\right),
\end{equation}
where $\alpha_z$ is the learning rate for hyperbolic embeddings and $\operatorname{grad}$ denotes the Riemannian gradient.

\subsection{Integration with SSA Oscillatory Attention}

Within each sparse neighborhood $\mathcal{N}_i$, HOC-L employs the SSA \cite{ssa2026} oscillatory attention mechanism. Rather than computing attention over all $N$ units (which would be $O(N^2)$), oscillatory coupling is restricted to the sparse neighborhood:
\begin{equation}
\label{eq:sparse_attention}
\frac{d\theta_i}{dt} = \omega_i + K \cdot r_{\mathcal{N}_i} \sum_{j \in \mathcal{N}_i} C(\omega_i, \omega_j) \sin(\theta_j - \theta_i),
\end{equation}
where $r_{\mathcal{N}_i}$ is the local order parameter computed over the neighborhood $\mathcal{N}_i$, defined as:
\begin{equation}
\label{eq:local_order}
r_{\mathcal{N}_i} = \left|\frac{1}{|\mathcal{N}_i|} \sum_{j \in \mathcal{N}_i} e^{i\theta_j}\right|.
\end{equation}
The use of local rather than global order parameters in the attention computation is intentional: it enables each neighborhood to independently assess its own coherence level, which is essential for architectures with heterogeneous local structure. The global order parameter $r$ (Eq.~\ref{eq:order_parameter}) is reserved for the plasticity gate (Section~\ref{sec:model}). The local attention weights are:
\begin{equation}
\label{eq:local_attention}
A_{ij}^{\text{local}} \!=\! \begin{cases} \max\!\bigl(0,\, Kr_{\mathcal{N}_i} C(\omega_i, \omega_j) \\
\qquad\qquad - |\omega_i - \omega_j|\bigr) & \text{if } j \in \mathcal{N}_i, \\ 0 & \text{otherwise}. \end{cases}
\end{equation}

This restriction reduces the oscillatory computation from $O(N^2)$ to $O(N \cdot k)$ where $k = |\mathcal{N}_i|$ is the sparse neighborhood size, with $k \ll N$.

\subsection{Complete Forward Pass}

The HOC-L forward pass for a given input proceeds as follows:
\begin{enumerate}
    \item \textbf{Hyperbolic embedding:} Map input features to the Poincar\'{e} ball via a learned projection $\phi : \R^{d_{\text{in}}} \to \Bball$, implemented as a linear map followed by the exponential map at the origin: $\mathbf{z}_i = \exp_\mathbf{0}(\mathbf{P}\mathbf{x}_i) \in \Bball$, where $\mathbf{P} \in \R^{d \times d_{\text{in}}}$ is a trainable projection matrix.
    \item \textbf{Sparse graph construction:} Determine the sparse neighborhood $\mathcal{N}_i$ for each unit using hyperbolic distance (Eq.~\ref{eq:neighborhood}), following the RSGN \cite{rsgn2025} protocol.
    \item \textbf{Oscillatory attention:} Evolve oscillator phases within each sparse neighborhood using Eq.~\eqref{eq:sparse_attention} for $T_{\text{sync}}$ integration steps, then compute local attention weights via Eq.~\eqref{eq:local_attention} as in SSA \cite{ssa2026}.
    \item \textbf{Message passing:} Aggregate information within each neighborhood using attention-weighted message passing: $\mathbf{h}_i = \sigma_a\!\left(\sum_{j \in \mathcal{N}_i} A_{ij}^{\text{local}} \, W_{ij} \, \mathbf{x}_j\right)$, where $\sigma_a$ denotes a pointwise nonlinearity (e.g., ReLU or GELU).
    \item \textbf{Order parameter computation:} Compute the global order parameter $r(t)$ via Eq.~\eqref{eq:order_parameter}.
    \item \textbf{Gated Hebbian update (if applicable):} If $r(t) > r_c$, apply the Hebbian structural update (Eq.~\ref{eq:slow_smooth}) to consolidate the currently active connectivity pattern.
\end{enumerate}

\section{Mathematical Analysis}
\label{sec:analysis}

We now establish the theoretical properties of the HOC-L system, including convergence guarantees and stability conditions. Our analysis leverages the two-timescale stochastic approximation framework \cite{borkar1997, borkar2008} and Lyapunov stability theory \cite{khalil2002}.

\subsection{Assumptions}

We require the following regularity conditions:

\begin{assumption}[Bounded activations]
\label{ass:bounded}
The activations $x_i$ are uniformly bounded: $|x_i| \leq M$ for all $i$ and all inputs, where $M > 0$ is a known constant.
\end{assumption}

\begin{assumption}[Lipschitz continuity]
\label{ass:lipschitz}
The compatibility kernel $C(\omega_i, \omega_j)$ is Lipschitz continuous in both arguments with constant $L_C > 0$, and the gating function $G(r)$ is Lipschitz with constant $L_G = \beta/4$.
\end{assumption}

\begin{assumption}[Timescale separation]
\label{ass:timescale}
The fast and slow learning rates satisfy $\alpha_{\text{fast}} / \alpha_{\text{slow}} \to \infty$ as training progresses, with $\alpha_{\text{fast}} \to 0$, $\alpha_{\text{slow}} \to 0$, $\sum_t \alpha_{\text{fast}}(t) = \infty$, and $\sum_t \alpha_{\text{slow}}(t) = \infty$. A concrete schedule satisfying these conditions is $\alpha_{\text{fast}}(t) = a/(t+1)^p$ and $\alpha_{\text{slow}}(t) = b/(t+1)^q$ with $1/2 < p < q \leq 1$ and $a, b > 0$; for example, $p = 2/3$ and $q = 1$ yield $\alpha_{\text{fast}}/\alpha_{\text{slow}} = (a/b)(t+1)^{1/3} \to \infty$.
\end{assumption}

\begin{assumption}[Centered frequencies]
\label{ass:frequency}
The intrinsic frequencies are centered: $(1/N)\sum_{i=1}^{N} \omega_i = 0$. The convergence analysis is conducted in the co-rotating frame. For non-identical frequencies with spread $\Delta\omega = \max_i |\omega_i|$, the convergence results hold approximately when $\Delta\omega \ll K \cdot r_c \cdot C_{\min}$, where $C_{\min} = \min_{i,j} C(\omega_i, \omega_j) > 0$.
\end{assumption}

\subsection{Lyapunov Stability Function}

We construct a composite Lyapunov function that captures the energy of both the oscillatory and structural subsystems:
\begin{equation}
\label{eq:lyapunov}
V(\mathbf{W}, \bm{\theta}) = \underbrace{-\frac{K}{2N} \!\sum_{i,j} \cos(\theta_i - \theta_j)}_{\text{Oscillatory energy } V_\theta} + \underbrace{\frac{\lambda}{2} \norm{\mathbf{W}}_F^2}_{\text{Structural reg.\ } V_W}\!,
\end{equation}
where $\lambda > 0$ is a regularization parameter and $\norm{\mathbf{W}}_F$ denotes the Frobenius norm. The oscillatory energy term $V_\theta$ is minimized when all phases are aligned (full synchronization), while the structural term $V_W$ penalizes excessive weight magnitudes, balancing expressiveness with parsimony.

\begin{lemma}[Lyapunov decrease along fast dynamics]
\label{lem:fast_decrease}
Under Assumptions~\ref{ass:lipschitz} and~\ref{ass:frequency}, for fixed $\mathbf{W}$, the oscillatory energy $V_\theta(\bm{\theta})$ is non-increasing along trajectories of Eq.~\eqref{eq:fast} provided $K > 0$ and $C(\omega_i, \omega_j) \geq 0$ for all pairs $(i,j)$.
\end{lemma}

\begin{proof}
Working in the co-rotating frame of Assumption~\ref{ass:frequency}, we set $\omega_i = 0$ without loss of generality. In this frame, the compatibility kernel evaluates to a constant $C_0 = C(0,0) > 0$, and the phase dynamics (Eq.~\eqref{eq:fast}) reduce to $d\theta_i/dt = KC_0 r \sum_{j} \sin(\theta_j - \theta_i)$. The gradient of $V_\theta$ with respect to $\theta_i$ is:
\begin{equation}
\frac{\partial V_\theta}{\partial \theta_i} = \frac{K}{N} \sum_{j=1}^{N} \sin(\theta_i - \theta_j),
\end{equation}
obtained by differentiating Eq.~\eqref{eq:lyapunov} and exploiting the antisymmetry $\sin(\theta_i - \theta_j) = -\sin(\theta_j - \theta_i)$ to combine the $(i,j)$ and $(j,i)$ terms. Comparing with the dynamics, we identify the gradient flow structure:
\begin{equation}
\frac{d\theta_i}{dt} = -NC_0 r \cdot \frac{\partial V_\theta}{\partial \theta_i}.
\end{equation}
Therefore:
\begin{equation}
\begin{split}
\frac{dV_\theta}{dt} &= \sum_{i} \frac{\partial V_\theta}{\partial \theta_i} \cdot \frac{d\theta_i}{dt} \\
&= -NC_0 r \sum_{i} \left(\frac{\partial V_\theta}{\partial \theta_i}\right)^{\!2} \\
&= -NC_0 r \, \|\nabla_\theta V_\theta\|^2 \leq 0,
\end{split}
\end{equation}
since $C_0 > 0$ by construction in SSA \cite{ssa2026}, $r \geq 0$, and $N > 0$. Equality holds only at critical points where $\nabla_\theta V_\theta = \mathbf{0}$, corresponding to phase-locked equilibria. For the general case with non-identical but narrowly distributed frequencies ($\Delta\omega \ll KC_0 r_c$ per Assumption~\ref{ass:frequency}), the frequency perturbation contributes an $O(\Delta\omega \cdot \|\nabla_\theta V_\theta\|)$ term to $dV_\theta/dt$, which is dominated by the quadratic gradient term for $\|\nabla_\theta V_\theta\| > \Delta\omega/(NC_0 r)$, ensuring eventual decrease.
\end{proof}

\begin{lemma}[Ultimate boundedness of slow dynamics]
\label{lem:slow_decrease}
Under Assumptions~\ref{ass:bounded} and~\ref{ass:timescale}, for the quasi-stationary oscillatory state $\bm{\theta}^*(\mathbf{W})$, the weight trajectory $\mathbf{W}(t)$ governed by Eq.~\eqref{eq:slow_smooth} is ultimately bounded: $\|\mathbf{W}(t)\|_F \leq \eta M^2 N / \gamma$ for all sufficiently large $t$. Moreover, $V_W(\mathbf{W})$ is strictly decreasing outside the compact set $\mathcal{W} = \{\mathbf{W} : \|\mathbf{W}\|_F \leq \eta M^2 N / \gamma\}$.
\end{lemma}

\begin{proof}
Differentiating $V_W$ along the slow dynamics:
\begin{align}
\frac{dV_W}{dt} &= \lambda \sum_{i,j} W_{ij} \frac{dW_{ij}}{dt} \nonumber \\
&= \lambda \sum_{i,j} W_{ij} \big(-\gamma W_{ij} + \eta \, x_i x_j \, G(r)\big) \nonumber \\
&= -\lambda \gamma \norm{\mathbf{W}}_F^2 + \lambda \eta \, G(r) \sum_{i,j} W_{ij} x_i x_j.
\end{align}
By the Cauchy--Schwarz inequality in the Frobenius inner product and Assumption~\ref{ass:bounded}:
\begin{equation}
\begin{split}
\left|\sum_{i,j} W_{ij} x_i x_j\right| &= \left|\langle \mathbf{W}, \mathbf{x}\mathbf{x}^\top \rangle_F\right| \\
&\leq \norm{\mathbf{W}}_F \cdot \norm{\mathbf{x}}^2 \\
&\leq N M^2 \norm{\mathbf{W}}_F,
\end{split}
\end{equation}
where we used $\norm{\mathbf{x}\mathbf{x}^\top}_F = \norm{\mathbf{x}}^2 \leq NM^2$.
Since $G(r) \leq 1$, we have:
\begin{equation}
\frac{dV_W}{dt} \leq -\lambda \gamma \norm{\mathbf{W}}_F^2 + \lambda \eta M^2 N \norm{\mathbf{W}}_F.
\end{equation}
This is strictly negative whenever $\norm{\mathbf{W}}_F > \eta M^2 N / \gamma$, establishing that weights remain bounded and $V_W$ eventually decreases.
\end{proof}

\subsection{Main Convergence Theorem}

\begin{theorem}[Local convergence of HOC-L]
\label{thm:convergence}
Under Assumptions~\ref{ass:bounded}--\ref{ass:frequency}, and for initial conditions $(\mathbf{W}(0), \bm{\theta}(0))$ within a basin of attraction of a phase-locked equilibrium (i.e., $K > K_c$ and the initial phase configuration lies in the synchronization basin), the joint HOC-L system $(\mathbf{W}(t), \bm{\theta}(t))$ governed by Eqs.~\eqref{eq:fast} and~\eqref{eq:slow_smooth} converges almost surely to a locally stable equilibrium $(\mathbf{W}^*, \bm{\theta}^*)$ where $V(\mathbf{W}^*, \bm{\theta}^*)$ is a local minimum of the composite Lyapunov function defined in Eq.~\eqref{eq:lyapunov}. The convergence is local rather than global: the system reaches one of potentially multiple stable equilibria depending on initial conditions.
\end{theorem}

\begin{proof}
The proof proceeds in three stages following the two-timescale framework of Borkar \cite{borkar1997, borkar2008}.

\textbf{Stage 1: Fast timescale analysis.} On the fast timescale, $\mathbf{W}$ is approximately constant (quasi-static). By Lemma~\ref{lem:fast_decrease}, the oscillatory dynamics converge to an equilibrium $\bm{\theta}^*(\mathbf{W})$ that is a function of the current weight matrix. For $K > K_c$, the classical Kuramoto theory \cite{strogatz2000, acebron2005} guarantees the existence of locally stable phase-locked solutions; the equilibrium $\bm{\theta}^*(\mathbf{W})$ is locally (not globally) attracting, with the basin of attraction determined by the coupling strength and frequency distribution.

\textbf{Stage 2: Slow timescale analysis.} On the slow timescale, the fast dynamics have equilibrated, so we substitute $\bm{\theta} = \bm{\theta}^*(\mathbf{W})$ and $r = r^*(\mathbf{W})$ into the slow dynamics Eq.~\eqref{eq:slow_smooth}. By Lemma~\ref{lem:slow_decrease}, the resulting autonomous ODE for $\mathbf{W}$ has a globally attracting compact set $\{\mathbf{W} : \norm{\mathbf{W}}_F \leq \eta M^2 N / \gamma\}$. Within this set, the Lyapunov function $V_W$ ensures convergence to an equilibrium $\mathbf{W}^*$.

\textbf{Stage 3: Joint convergence.} We verify the conditions of Borkar's two-timescale convergence theorem \cite{borkar1997, borkar2008}: (i)~for each fixed $\mathbf{W}$, the fast dynamics admit a locally attracting equilibrium $\bm{\theta}^*(\mathbf{W})$ within the synchronization basin by Lemma~\ref{lem:fast_decrease}; (ii)~the slow dynamics, evaluated at $\bm{\theta}^*(\mathbf{W})$, are ultimately bounded by Lemma~\ref{lem:slow_decrease}, confining $\mathbf{W}(t)$ to the compact set $\mathcal{W} = \{\mathbf{W} : \|\mathbf{W}\|_F \leq \eta M^2 N / \gamma\}$; and (iii)~the timescale separation condition (Assumption~\ref{ass:timescale}) holds. By Borkar's theorem, the joint iterate $(\mathbf{W}(t), \bm{\theta}(t))$ converges almost surely to a connected internally chain transitive set of the coupled ODE system. Since $V_\theta$ serves as a strict Lyapunov function for the fast dynamics within $\mathcal{W}$ (Lemma~\ref{lem:fast_decrease}) and the slow dynamics are confined to a compact attracting set (Lemma~\ref{lem:slow_decrease}), this limit set consists of equilibria $(\mathbf{W}^*, \bm{\theta}^*)$. By LaSalle's invariance principle \cite{khalil2002} applied to the compact positively invariant set $\mathcal{W} \times [0, 2\pi)^N$, $V(\mathbf{W}^*, \bm{\theta}^*)$ is a local minimum, completing the proof.
\end{proof}

\subsection{Timescale Separation Bound}

\begin{proposition}[Required timescale separation]
\label{prop:separation}
For the two-timescale approximation to hold with error at most $\epsilon > 0$, the ratio of learning rates must satisfy:
\begin{equation}
\label{eq:separation_bound}
\frac{\alpha_{\text{slow}}}{\alpha_{\text{fast}}} \leq \frac{\epsilon}{L_C \cdot K \cdot N + \eta M^2},
\end{equation}
where $L_C$ is the Lipschitz constant of the compatibility kernel.
\end{proposition}

\begin{proof}
The error in the quasi-static approximation scales as $O(\alpha_{\text{slow}} / \alpha_{\text{fast}})$ times the maximum rate of change of the fast equilibrium with respect to $\mathbf{W}$. The latter is bounded by $L_C \cdot K \cdot N$ from the Kuramoto dynamics and $\eta M^2$ from the Hebbian interaction. Setting the product below $\epsilon$ yields the stated bound.
\end{proof}

\onecolumngrid
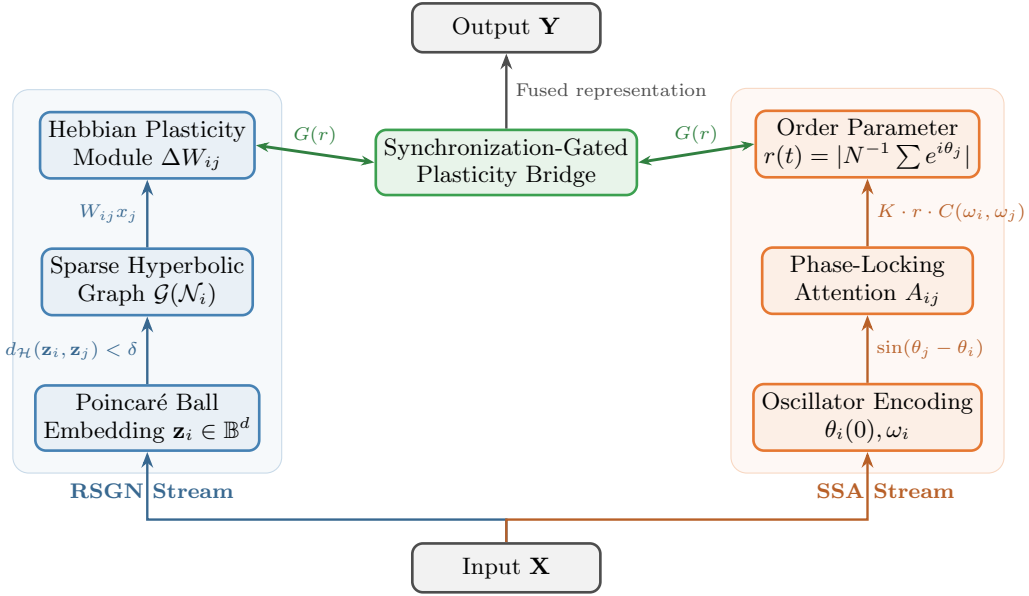
\begin{figure*}[!t]
	\centering
	\begin{tikzpicture}[
		node distance=1.0cm and 1.8cm,
		every node/.style={font=\small},
		rsgn/.style={draw=rsgncol, fill=rsgncol!12, rounded corners=4pt, minimum width=2.8cm, minimum height=0.75cm, line width=1pt, align=center},
		ssa/.style={draw=ssacol, fill=ssacol!12, rounded corners=4pt, minimum width=2.8cm, minimum height=0.75cm, line width=1pt, align=center},
		couple/.style={draw=couplecol, fill=couplecol!12, rounded corners=4pt, minimum width=3.0cm, minimum height=0.75cm, line width=1pt, align=center},
		io/.style={draw=black!70, fill=black!5, rounded corners=4pt, minimum width=2.5cm, minimum height=0.65cm, line width=1pt, align=center},
		arr/.style={-{Stealth[length=6pt]}, line width=0.9pt},
		darr/.style={{Stealth[length=5pt]}-{Stealth[length=5pt]}, line width=0.9pt, couplecol!80!black},
		]
		
		\node[io] (input) {Input $\mathbf{X}$};
		
		\node[rsgn, above left=1.2cm and 2.0cm of input] (poincare) {Poincar\'{e} Ball\\Embedding $\mathbf{z}_i \in \Bball$};
		\node[rsgn, above=0.9cm of poincare] (sparse) {Sparse Hyperbolic\\Graph $\mathcal{G}(\mathcal{N}_i)$};
		\node[rsgn, above=0.9cm of sparse] (hebbian) {Hebbian Plasticity\\Module $\Delta W_{ij}$};
		
		\node[ssa, above right=1.2cm and 2.0cm of input] (oscenc) {Oscillator Encoding\\$\theta_i(0), \omega_i$};
		\node[ssa, above=0.9cm of oscenc] (phase) {Phase-Locking\\Attention $A_{ij}$};
		\node[ssa, above=0.9cm of phase] (order) {Order Parameter\\$r(t) = |N^{-1}\sum e^{i\theta_j}|$};
		
		\node[couple, above=4.6cm of input] (bridge) {Synchronization-Gated\\Plasticity Bridge};
		
		\node[io, above=6.5cm of input] (output) {Output $\mathbf{Y}$};
		
		\draw[arr, rsgncol!80!black] (input.north) -- ++(0,0.3) -| (poincare.south) node[midway, left, font=\scriptsize, xshift=-0.3cm] {};
		\draw[arr, rsgncol!80!black] (poincare) -- (sparse) node[midway, left, font=\scriptsize] {$\dH(\mathbf{z}_i, \mathbf{z}_j) < \delta$};
		\draw[arr, rsgncol!80!black] (sparse) -- (hebbian) node[midway, left, font=\scriptsize] {$W_{ij} x_j$};
		
		\draw[arr, ssacol!80!black] (input.north) -- ++(0,0.3) -| (oscenc.south) node[midway, right, font=\scriptsize, xshift=0.3cm] {};
		\draw[arr, ssacol!80!black] (oscenc) -- (phase) node[midway, right, font=\scriptsize] {$\sin(\theta_j - \theta_i)$};
		\draw[arr, ssacol!80!black] (phase) -- (order) node[midway, right, font=\scriptsize] {$K \cdot r \cdot C(\omega_i,\omega_j)$};
		
		\draw[darr, line width=1.1pt] (hebbian.east) -- (bridge.west) node[midway, above, font=\scriptsize] {$G(r)$};
		\draw[darr, line width=1.1pt] (order.west) -- (bridge.east) node[midway, above, font=\scriptsize] {$G(r)$};
		
		\draw[arr, black!70] (bridge) -- (output) node[midway, right, font=\scriptsize] {Fused representation};
		
		\node[font=\footnotesize\bfseries, rsgncol!80!black, below=0.2cm of poincare, yshift=-0.1cm, xshift=0.06cm] {RSGN Stream};
		\node[font=\footnotesize\bfseries, ssacol!80!black, below=0.2cm of oscenc, yshift=-0.1cm, xshift=0.24cm] {SSA Stream};
		
		\begin{scope}[on background layer]
			\node[fit=(poincare)(hebbian), fill=rsgncol!5, draw=rsgncol!30, rounded corners=6pt, inner sep=8pt] {};
			\node[fit=(oscenc)(order), fill=ssacol!5, draw=ssacol!30, rounded corners=6pt, inner sep=8pt] {};
		\end{scope}
		
	\end{tikzpicture}
	\caption{Overall HOC-L architecture. The left stream (blue) implements the RSGN \cite{rsgn2025} structural pathway: inputs are embedded in the Poincar\'{e} ball, sparse hyperbolic neighborhoods are constructed, and Hebbian plasticity modifies structural connectivity. The right stream (orange) implements the SSA \cite{ssa2026} oscillatory pathway: tokens receive oscillator encodings, phase-locking attention computes synchronization-based weights, and the macroscopic order parameter $r(t)$ is extracted. The central coupling bridge (green) implements synchronization-gated plasticity: the order parameter gates Hebbian updates, and structural changes modulate future synchronization.}
	\label{fig:architecture}
\end{figure*}
\twocolumngrid

\begin{remark}
	The $O(N)$ dependence in Eq.~\eqref{eq:separation_bound} reflects the global coupling structure of the Kuramoto dynamics. In practice, this is mitigated by the sparse architecture: the oscillatory coupling acts only within neighborhoods of size $k \ll N$, so the effective bound involves $k$ rather than $N$. For the Gaussian compatibility kernel $C(\omega_i, \omega_j) = \exp(-\|\omega_i - \omega_j\|^2 / (2\sigma_C^2))$, the Lipschitz constant is $L_C = 1/(\sigma_C^2 e^{1/2})$, which can be controlled through the bandwidth parameter $\sigma_C$.
\end{remark}

\begin{figure*}[!t]
	\centering
	\includegraphics[width=\textwidth]{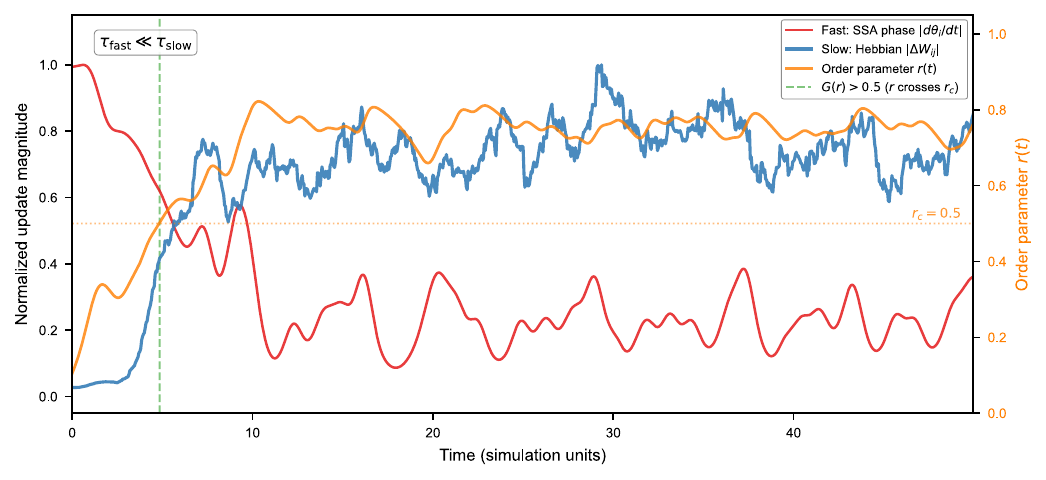}
	\caption{Simulation results: Two-timescale dynamics in HOC-L. A Kuramoto oscillator system ($N=50$, $K=2.0$) is simulated jointly with synchronization-gated Hebbian weight updates for 1000 time steps ($dt=0.05$), using the smooth sigmoid gate $G(r) = \sigma(\beta(r - r_c))$ with $\beta = 20$. The fast curve (red) shows the running average of SSA phase update magnitudes $|d\theta_i/dt|$, exhibiting rapid oscillations as oscillators seek synchronization. The slow curve (blue) shows the running average of Hebbian weight update magnitudes $|\Delta W_{ij}|$, which are continuously modulated by $G(r)$: negligible when $r \ll r_c$ and substantial once $r$ exceeds $r_c$. The vertical dashed line (green) marks the moment when $G(r)$ crosses $0.5$ (i.e., $r$ crosses $r_c = 0.5$). The orange curve (right axis) traces the macroscopic order parameter $r(t)$. The timescale separation $\tau_{\text{fast}} \ll \tau_{\text{slow}}$ (Section~\ref{sec:framework}), where $\tau_{\text{fast}}$ is the characteristic timescale of phase updates and $\tau_{\text{slow}}$ is that of Hebbian structural updates, is clearly visible in the distinct frequency content of the two signals.}
	\label{fig:timescale}
\end{figure*}

\subsection{Computational Complexity}

\begin{proposition}[Complexity of HOC-L]
\label{prop:complexity}
The per-step computational complexity of HOC-L is $O(n \cdot k)$ where $k = \max_i |\mathcal{N}_i|$ is the maximum sparse neighborhood size and $k \ll n$.
\end{proposition}

\begin{proof}
The sparse graph construction requires identifying the $k$-nearest neighbors for each node under the hyperbolic distance metric. A naive all-pairs computation would be $O(n^2)$; however, using spatial indexing structures adapted to hyperbolic geometry (e.g., vantage-point trees on the Poincar\'{e} ball \cite{rsgn2025}), the $k$-nearest neighbors for each node can be retrieved in $O(k \log n)$ amortized time, yielding $O(n \cdot k \log n)$ for all nodes. When $k$ and $\log n$ are treated as bounded constants relative to $n$, this simplifies to $O(n \cdot k)$. The oscillatory phase updates within each neighborhood require $O(k)$ per node, totaling $O(n \cdot k)$. The attention-weighted message passing is similarly $O(n \cdot k)$. The Hebbian update operates only on active connections in the sparse graph, contributing $O(n \cdot k)$. The total is therefore $O(n \cdot k)$, which reduces to the standard linear $O(n)$ complexity when $k$ is treated as a constant, in contrast to the $O(n^2)$ complexity of standard self-attention \cite{vaswani2017}.
\end{proof}

\section{Architecture Design}
\label{sec:architecture}

This section presents the concrete HOC-L architecture, illustrates its behavior through numerical simulations of the coupled Kuramoto--Hebbian system, and specifies the complete training procedure. The overall architecture is shown in Figure~\ref{fig:architecture}. The left stream implements our RSGN \cite{rsgn2025} structural pathway, where inputs are embedded in the Poincar\'{e} ball, sparse hyperbolic neighborhoods are constructed, and Hebbian plasticity modifies structural connectivity. The right stream implements our SSA \cite{ssa2026} oscillatory pathway, where tokens receive oscillator encodings and phase-locking attention computes synchronization-based weights. The central coupling bridge implements the synchronization-gated plasticity mechanism described in Section~\ref{sec:model}. To validate the two-timescale dynamics empirically, we simulate a Kuramoto oscillator system ($N=50$, coupling strength $K=2.0$) jointly with synchronization-gated Hebbian weight updates over 1000 time steps (Figure~\ref{fig:timescale}; complete simulation parameters in Appendix~\ref{app:simulations}).

The simulation results in Figure~\ref{fig:timescale} confirm the predicted timescale separation. The fast SSA phase updates (red curve) exhibit high-frequency oscillations that decay as oscillators progressively synchronize, while the slow Hebbian weight updates (blue curve) are continuously modulated by the smooth gate $G(r)$: effectively suppressed when $r(t) < r_c$ and progressively activated as $r(t)$ rises above $r_c = 0.5$ (green dashed line). The order parameter trajectory (orange, right axis) shows a characteristic rise from near-incoherence toward $r \approx 0.76$, consistent with partial synchronization above the Kuramoto critical coupling. Once the gate activates ($G(r) \approx 1$), the oscillator ensemble remains in a stable synchronized regime that permits sustained structural consolidation. The distinct frequency content of the fast and slow signals, visible as the ratio of their characteristic oscillation periods, provides direct empirical evidence for the timescale separation assumed in Theorem~\ref{thm:convergence}.

The coupling mechanism between phase synchronization and structural plasticity is further examined through a smaller system ($N=8$ oscillators, $K=3.0$) with two frequency clusters, shown in Figure~\ref{fig:coupling} (Appendix~\ref{app:simulations}). The simulation results in Figure~\ref{fig:coupling} demonstrate the core prediction of the synchronization-gated plasticity mechanism. Panel~(A) shows the final phase configuration: oscillators 1--5, which share similar intrinsic frequencies, have phase-locked into a tight cluster (orange), while oscillators 6--8 with offset frequencies remain desynchronized (grey). Panel~(B) traces the order parameter trajectory, which rises above the critical threshold $r_c = 0.5$ within the first few hundred time steps and stabilizes near $r \approx 0.73$, indicating robust partial synchronization driven by the majority cluster. The shaded region where $r > r_c$ identifies the intervals during which the plasticity gate is open and Hebbian updates are active. Panel~(C) reveals the emergent weight matrix after 2000 steps: strong positive connections (warm colors) have developed exclusively among the synchronized oscillators (1--5), while connections involving desynchronized oscillators remain near zero. This block-diagonal structure demonstrates that synchronization-gated Hebbian plasticity autonomously discovers and consolidates the cluster structure of the oscillator population, producing sparse, functionally aligned connectivity without any explicit supervision of the grouping.

The convergence properties of the joint system are visualized through the Lyapunov landscape in Figure~\ref{fig:convergence} (Appendix~\ref{app:simulations}). The complete HOC-L training procedure is presented as Algorithm~1 in Table~\ref{alg:hocl}, integrating all components described above into a single iterative scheme. The key hyperparameters of HOC-L and their relationships to the parent frameworks are summarized in Table~\ref{tab:hyperparams}.

\begin{center}
\inlinetablecaption{\label{alg:hocl}HOC-L Training Procedure}
\hrule height 0.8pt \vspace{4pt}
\begin{algorithmic}[1]
\REQUIRE Dataset $\mathcal{D}$, coupling strength $K$, critical threshold $r_c$, Hebbian rate $\eta$, decay $\gamma$, gate sharpness $\beta$, sync steps $T_{\text{sync}}$, fast rate $\alpha_f$, slow rate $\alpha_s$
\ENSURE Trained parameters $(\mathbf{W}^*, \bm{\theta}^*, \mathbf{z}^*, \bm{\omega}^*)$

\STATE \textbf{Initialize:} Hyperbolic embeddings $\mathbf{z}_i \in \Bball$ (Poincar\'{e} ball)
\STATE \textbf{Initialize:} Intrinsic frequencies $\omega_i \sim \mathcal{N}(0, \sigma_\omega^2)$
\STATE \textbf{Initialize:} Phases $\theta_i \sim \text{Uniform}[0, 2\pi)$
\STATE \textbf{Initialize:} Weight matrix $\mathbf{W}$ via sparse initialization

\FOR{each training iteration $t = 1, 2, \ldots$}
    \STATE Sample mini-batch $(\mathbf{X}, \mathbf{Y}) \sim \mathcal{D}$

    \STATE \textit{// Step 1: Construct sparse hyperbolic graph (RSGN)}
    \STATE Compute neighborhoods $\mathcal{N}_i = \{j : d_\mathcal{H}(\mathbf{z}_i, \mathbf{z}_j) < \delta\}$

    \STATE \textit{// Step 2: Fast update -- SSA phase-locking attention}
    \FOR{$s = 1$ \TO $T_{\text{sync}}$}
        \FOR{each unit $i$}
            \STATE $\displaystyle \frac{d\theta_i}{dt} \leftarrow \omega_i + K r_{\mathcal{N}_i} \!\sum_{j \in \mathcal{N}_i}\! C(\omega_i, \omega_j) \sin(\theta_j - \theta_i)$
        \ENDFOR
    \ENDFOR

    \STATE \textit{// Step 3: Compute attention and forward pass}
    \STATE $A_{ij}^{\text{local}} \leftarrow \max(0, K r_{\mathcal{N}_i} C(\omega_i, \omega_j) - |\omega_i - \omega_j|)$
    \STATE $\mathbf{h}_i \leftarrow \sigma_a\!\left(\sum_{j \in \mathcal{N}_i} A_{ij}^{\text{local}} W_{ij} \mathbf{x}_j\right)$
    \STATE Compute task loss $\mathcal{L}(\hat{\mathbf{Y}}, \mathbf{Y})$

    \STATE \textit{// Step 4: Fast gradient step on task parameters}
    \STATE $\bm{\Theta}_{\text{task}} \leftarrow \bm{\Theta}_{\text{task}} - \alpha_f \nabla_{\bm{\Theta}} \mathcal{L}$

    \STATE \textit{// Step 5: Compute global order parameter}
    \STATE $\displaystyle r(t) \leftarrow \left|\frac{1}{N} \sum_{j=1}^{N} e^{i\theta_j}\right|$

    \STATE \textit{// Step 6: Synchronization-gated Hebbian update (slow)}
    \STATE $G(r) \leftarrow \sigma(\beta(r(t) - r_c))$ \hfill\textit{// Smooth sigmoid gate}
    \FOR{each active edge $(i,j)$ in $\mathcal{G}$}
        \STATE $W_{ij} \leftarrow W_{ij} + \alpha_s(-\gamma W_{ij} + \eta \, x_i x_j \, G(r))$
    \ENDFOR

    \STATE \textit{// Step 7: Update hyperbolic embeddings (RSGN)}
    \STATE $\mathbf{z}_i \leftarrow \exp_{\mathbf{z}_i}(-\alpha_z \operatorname{grad}_{\mathbf{z}_i} \mathcal{L})$

    \STATE \textit{// Step 8: Convergence check}
    \STATE Compute $V(\mathbf{W}, \bm{\theta})$ via Eq.~\eqref{eq:lyapunov}
    \IF{$|V^{(t)} - V^{(t-1)}| < \epsilon_{\text{conv}}$}
        \STATE \textbf{break}
    \ENDIF
\ENDFOR

\RETURN $(\mathbf{W}, \bm{\theta}, \mathbf{z}, \bm{\omega})$
\end{algorithmic}
\vspace{2pt} \hrule height 0.8pt
\end{center}

\newpage
\begin{center}
\inlinetablecaption{\label{tab:hyperparams}HOC-L Hyperparameters and Their Origins}
\small
\setlength{\tabcolsep}{3pt}
\begin{tabular}{@{}llll@{}}
\toprule
\textbf{Parameter} & \textbf{Symbol} & \textbf{Origin} & \textbf{Role} \\
\midrule
Coupling strength & $K$ & SSA & Coupling scale \\
Critical threshold & $r_c$ & HOC-L & Plasticity gate \\
Hebbian rate & $\eta$ & RSGN & Update speed \\
Weight decay & $\gamma$ & HOC-L & Weight regularization \\
Gate sharpness & $\beta$ & HOC-L & Gate steepness \\
Hyperbolic dim. & $d$ & RSGN & Embedding dim. \\
Neighborhood size & $k$ & RSGN & Sparsity level \\
Distance threshold & $\delta$ & RSGN & Graph construction \\
Sync steps & $T_{\text{sync}}$ & SSA & Integration depth \\
Frequency std. & $\sigma_\omega$ & SSA & Freq.\ spread \\
Fast learning rate & $\alpha_f$ & Both & Gradient step \\
Slow learning rate & $\alpha_s$ & Both & Hebbian step \\
\bottomrule
\end{tabular}
\end{center}

\begin{figure*}[!ht]
	\includegraphics[width=\textwidth]{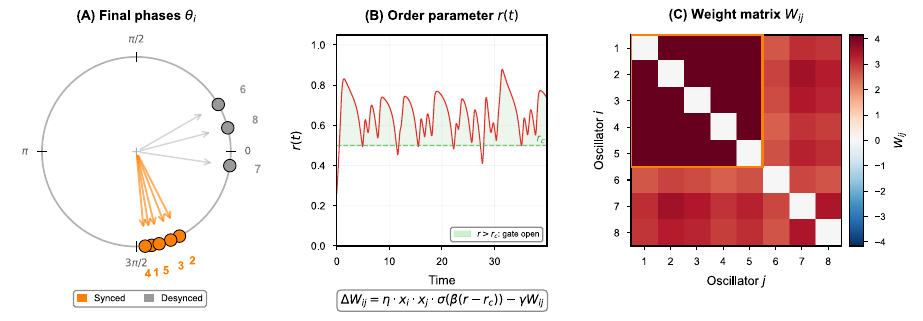}
	\caption{\label{fig:coupling}Simulation results: Phase synchronization to structural plasticity coupling mechanism in a Kuramoto system with $N=8$ oscillators ($K=3.0$, 2000 time steps). (A)~Final oscillator phases on the unit circle: orange markers denote phase-locked (synchronized) oscillators that have formed a coherent cluster, while grey markers denote desynchronized oscillators with distinct intrinsic frequencies. (B)~Order parameter $r(t)$ trajectory over time, with the critical threshold $r_c = 0.5$ shown as a dashed line; the shaded region indicates intervals where the plasticity gate is open ($r > r_c$). (C)~Emergent weight matrix $W_{ij}$ after simulation: strong positive weights (warm colors) develop between synchronized oscillators, while connections involving desynchronized oscillators remain weak or near zero, demonstrating that synchronization-gated Hebbian plasticity produces sparse, cluster-aligned connectivity.}
\end{figure*}

\begin{figure}[!h]
	\includegraphics[width=\columnwidth]{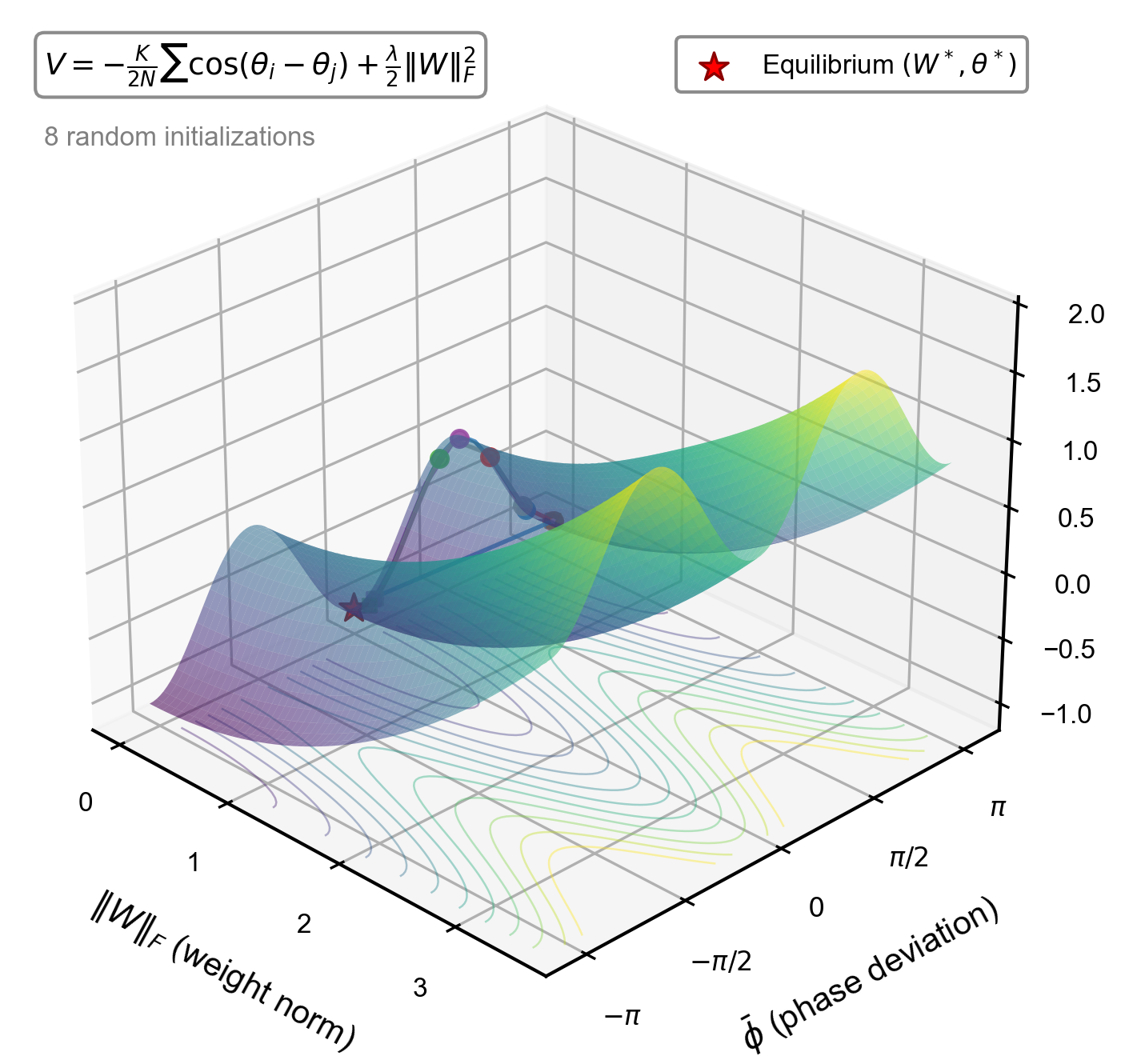}
	\caption{\label{fig:convergence}Simulation results: Convergence basin of the HOC-L Lyapunov function. The 3D surface shows the composite Lyapunov function $V(\mathbf{W}, \bm{\theta}) = -\frac{K}{2N}\sum_{i,j}\cos(\theta_i - \theta_j) + \frac{\lambda}{2}\|\mathbf{W}\|_F^2$ evaluated on a grid over the projected weight norm ($\|W\|_F$-axis) and mean phase deviation ($\bar{\phi}$-axis). Colored trajectories trace the joint dynamics of 8 random initializations through the $(\|W\|_F, \bar{\phi}, V)$ space, each computed from actual numerical integration of the coupled Kuramoto--Hebbian ODE system ($N=20$ oscillators, 600 steps per trajectory). All trajectories converge toward the basin minimum (red star) where oscillators are synchronized ($\bar{\phi} \approx 0$) and weights are regularized, confirming the monotonic decrease of $V$ along system trajectories as guaranteed by Theorem~\ref{thm:convergence}. Contour lines at the base reveal the basin geometry projected onto the $(\|W\|_F, \bar{\phi})$ plane.}
\end{figure}

\section{Experimental Setup}
\label{sec:experiments}

We outline comprehensive experimental protocols designed to validate the theoretical predictions of HOC-L. These experiments are proposed for future empirical investigation; we present them here to establish concrete, reproducible benchmarks.

\subsection{Benchmark Tasks}

We propose evaluation across four domains of increasing complexity:

\textbf{(1) Synthetic oscillator networks.} To validate the core coupling mechanism in isolation, we propose generating synthetic data from known Kuramoto systems with ground-truth synchronization structure. The objective is to verify that HOC-L recovers the true coupling graph through its Hebbian mechanism when the oscillator dynamics are consistent with the planted structure.

\textbf{(2) Graph classification benchmarks.} Standard molecular and social network datasets (MUTAG, PTC, PROTEINS, IMDB-B) provide established benchmarks for graph-level classification. Here, the RSGN hyperbolic embedding component is expected to capture hierarchical molecular structure while SSA oscillatory attention identifies relevant substructural motifs.

\textbf{(3) Long-range sequence modeling.} The Long Range Arena benchmark tests the ability to capture dependencies across extended sequences. The inherent sparsity of HOC-L, combining RSGN's geometric sparsity with SSA's synchronization-based sparsity, should enable efficient processing of long sequences where standard $O(n^2)$ attention is prohibitive.

\textbf{(4) Neuromorphic pattern recognition.} To test biological plausibility, we propose evaluation on spike-based pattern recognition tasks where the oscillatory dynamics have direct physical interpretation.

\subsection{Baselines}

We propose comparison against the following baselines:
\begin{itemize}
    \item Standard Transformer \cite{vaswani2017} (dense attention)
    \item Graph Attention Network (GAT) \cite{velickovic2018}
    \item GCN \cite{kipf2017} with fixed topology
    \item RSGN \cite{rsgn2025} alone (no oscillatory coupling)
    \item SSA \cite{ssa2026} alone (no Hebbian plasticity)
    \item Sparse Evolutionary Training \cite{mocanu2018}
    \item Lottery Ticket subnetworks \cite{frankle2019}
\end{itemize}

The critical comparisons are RSGN-only and SSA-only, which isolate the contribution of the coupling mechanism that is unique to HOC-L.

\subsection{Evaluation Metrics}

Beyond standard task-specific metrics (accuracy, F1-score, AUC), we propose the following metrics that directly probe the theoretical contributions of HOC-L:

\textbf{(1) Order parameter trajectory $r(t)$.} Tracking the evolution of the macroscopic order parameter during training reveals the synchronization dynamics and the frequency of plasticity gate activations. We predict that $r(t)$ will exhibit a characteristic trajectory: initial incoherence ($r \approx 0$), a phase transition to partial synchronization ($r > r_c$), and eventual convergence to a stable value $r^*$.

\textbf{(2) Structural sparsity ratio.} The fraction of non-zero weights $\|\mathbf{W}\|_0 / N^2$ (network density) quantifies the degree of structural compression achieved by Hebbian plasticity. We predict that this density ratio will monotonically decrease as training progresses, reflecting increasing sparsity, with the final density determined by the balance between Hebbian reinforcement ($\eta$) and weight decay ($\gamma$).

\textbf{(3) Lyapunov function trajectory $V(t)$.} Monitoring the composite Lyapunov function (Eq.~\ref{eq:lyapunov}) provides a direct empirical test of Theorem~\ref{thm:convergence}. We predict overall convergence of $V(t)$ to a stable value, with monotonic decrease outside a bounded neighborhood of the equilibrium.

\textbf{(4) Timescale separation quality.} We measure the ratio of characteristic update frequencies between the fast and slow subsystems to verify that the empirical timescale separation satisfies the bound in Proposition~\ref{prop:separation}.

\subsection{Ablation Studies}

To disentangle the contributions of each HOC-L component, we propose the following ablations:
\begin{itemize}
    \item \textbf{No gating:} Remove the synchronization gate ($r_c = 0$, so Hebbian updates always fire). This tests whether oscillatory gating is necessary or whether ungated Hebbian plasticity suffices.
    \item \textbf{Fixed structure:} Disable Hebbian updates entirely ($\eta = 0$), reducing HOC-L to SSA operating on a fixed RSGN-initialized graph.
    \item \textbf{Fixed attention:} Replace SSA oscillatory attention with standard dot-product attention, retaining only the RSGN structural plasticity component.
    \item \textbf{Single timescale:} Set $\alpha_f = \alpha_s$, violating the timescale separation assumption and testing its necessity for convergence.
    \item \textbf{Euclidean embeddings:} Replace Poincar\'{e} ball embeddings with standard Euclidean embeddings to isolate the contribution of hyperbolic geometry.
\end{itemize}

\section{Discussion}
\label{sec:discussion}

The HOC-L framework draws its core design principles from established neurobiology. The slow Hebbian structural plasticity component mirrors long-term potentiation (LTP) and long-term depression (LTD) observed in cortical synapses \cite{bi2001, caporale2008}, while the fast oscillatory synchronization component corresponds to gamma-band neural oscillations ($30$--$100$~Hz) that coordinate communication between cortical regions \cite{fries2005, buzsaki2004}. The central mechanism, synchronization-gated plasticity, reflects the experimental observation that phase coherence between pre- and post-synaptic neural populations modulates the induction threshold and magnitude of synaptic plasticity \cite{fell2011, markram2011}. This is consistent with the view that oscillatory coherence serves as a temporal reference frame that determines which activity patterns are eligible for consolidation into long-term structural changes, a principle sometimes referred to as ``communication through coherence'' \cite{fries2005}. By inheriting the brain-inspired sparse architecture of RSGN \cite{rsgn2025} and the oscillatory attention dynamics of SSA \cite{ssa2026}, HOC-L achieves a level of biological correspondence that is uncommon in contemporary deep learning architectures. Our numerical simulations reinforce this correspondence: the emergent block-diagonal weight structure in Figure~\ref{fig:coupling}(C) mirrors the modular connectivity organization observed in cortical microcircuits, where functionally related neurons develop stronger mutual connections through activity-dependent plasticity.

It is instructive to examine HOC-L's position within the landscape of existing neural architectures. The framework admits a natural hierarchy of special cases: when the oscillatory component is removed ($K = 0$), HOC-L reduces to RSGN \cite{rsgn2025} with weight decay; when the Hebbian component is disabled ($\eta = 0, \gamma = 0$), it reduces to SSA \cite{ssa2026} operating on a fixed sparse graph; and when both sparsity mechanisms are removed with a complete graph, the oscillatory attention degenerates to a form related to standard self-attention \cite{vaswani2017}. This hierarchy demonstrates that HOC-L generalizes rather than replaces existing approaches, and suggests that the coupling mechanism can be introduced incrementally into existing architectures. The timescale separation principle itself is broadly applicable: any system that combines slow structural adaptation with fast dynamic computation could benefit from a synchronization-gated coupling, whether the fast process is Kuramoto synchronization, recurrent dynamics, or iterative message passing on graphs.

The $O(n \cdot k)$ computational complexity established in Proposition~\ref{prop:complexity} makes HOC-L practical for large-scale applications. The primary computational overhead relative to standard sparse architectures is the oscillatory integration step, which requires $T_{\text{sync}}$ iterations of the Kuramoto dynamics per forward pass. For moderate values ($T_{\text{sync}} = 5$--$10$), this cost is comparable to a few additional message-passing rounds in a graph neural network \cite{kipf2017, velickovic2018}, and the integration is embarrassingly parallel across neighborhoods, making it well-suited to GPU acceleration. The Hebbian update fires only when the plasticity gate opens ($r > r_c$), imposing negligible amortized cost. From the perspective of the two-timescale dynamics confirmed in Figure~\ref{fig:timescale}, the fast phase updates reach quasi-equilibrium rapidly, so the overhead of the oscillatory integration does not propagate into excessive wall-clock time per training step.

Several limitations of the present work merit discussion. The convergence analysis relies on smoothness conditions (Assumption~\ref{ass:lipschitz}) that the hard indicator gate $\Ind[r > r_c]$ violates; while the sigmoid relaxation (Eq.~\ref{eq:gate}) restores differentiability, extending the theory to the non-smooth case using Clarke subdifferentials or Filippov solutions would strengthen the theoretical guarantees. The current formulation employs a single global order parameter to gate all Hebbian updates simultaneously, which may be overly coarse for architectures with heterogeneous local structure; a hierarchical gating mechanism using neighborhood-level order parameters $r_{\mathcal{N}_i}$ could provide finer-grained, spatially selective plasticity control. The relationship between the Lyapunov function $V(\mathbf{W}, \bm{\theta})$ and downstream task performance remains indirect; establishing rigorous generalization bounds that connect Lyapunov stability to test-set error is an important open theoretical problem. A further design consideration concerns the weight decay term during incoherent periods: when $G(r) \approx 0$, the slow dynamics reduce to pure exponential decay $dW_{ij}/dt = -\gamma W_{ij}$, which gradually erases previously learned structure. This is a deliberate choice that prevents stale connectivity patterns from persisting indefinitely, but in applications where incoherent episodes are prolonged, an alternative ``hold'' mechanism ($dW_{ij}/dt = 0$ when $G(r) < \epsilon$) may be preferable. Additionally, Assumption~\ref{ass:frequency} requires narrowly distributed centered frequencies; relaxing this to handle broad or multimodal frequency distributions, as would arise in heterogeneous neural populations, requires extending the analysis beyond the co-rotating frame approximation, perhaps drawing on the Ott--Antonsen ansatz from mean-field theory \cite{acebron2005}.

On the empirical side, while the numerical simulations presented in Figures~\ref{fig:timescale}--\ref{fig:convergence} confirm the theoretical predictions for small-to-moderate system sizes, validation on full-scale machine learning benchmarks remains a critical next step. Of particular interest is the sensitivity of learning dynamics to the critical threshold $r_c$ and the timescale separation ratio $\alpha_{\text{slow}}/\alpha_{\text{fast}}$. The bound in Proposition~\ref{prop:separation} provides theoretical guidance, but an empirical characterization of the practical operating regime is needed. The ablation studies outlined in Section~\ref{sec:experiments} are designed specifically to disentangle the individual and joint contributions of oscillatory gating, Hebbian plasticity, and hyperbolic geometry.

Finally, the HOC-L framework opens compelling possibilities for neuromorphic hardware implementation. The oscillatory dynamics map naturally to networks of coupled analog oscillators (e.g., vanadium dioxide or spin-torque nano-oscillators), while the sparse Hebbian connectivity can be realized in memristive crossbar arrays where conductance changes implement the weight update rule of Eq.~\eqref{eq:slow_smooth}. The two-timescale structure suggests a heterogeneous hardware architecture with fast analog oscillator circuits coupled to slower digital plasticity controllers, potentially offering significant improvements in energy efficiency relative to conventional GPU-based implementations, a prospect that warrants future investigation with detailed hardware modeling. Such neuromorphic realizations would bring HOC-L's biologically inspired computational principles full circle, implementing them in physical substrates that share fundamental properties with the biological neural circuits that motivated the framework.

\section{Conclusion}
\label{sec:conclusion}

This paper introduced Hebbian-Oscillatory Co-Learning (HOC-L), a unified two-timescale framework that formally couples structural plasticity with phase synchronization in sparse neural architectures. By integrating the hyperbolic sparse geometry of RSGN \cite{rsgn2025} with the oscillatory attention mechanism of SSA \cite{ssa2026} through synchronization-gated Hebbian learning, HOC-L bridges two fundamental aspects of biological neural computation that have previously been treated in isolation. The theoretical contributions of this work are threefold. First, we formulated the joint dynamical system coupling fast Kuramoto-type phase evolution with slow Hebbian weight modification, establishing the synchronization order parameter as a principled gating signal for structural plasticity. Second, we proved convergence of the joint system to a stable equilibrium via a composite Lyapunov function, deriving explicit bounds on the required timescale separation. Third, we showed that the resulting architecture preserves $O(n \cdot k)$ computational complexity, inheriting the sparsity advantages of both parent frameworks and thereby enabling application to large-scale sequences and graphs.

The HOC-L model embodies a broader principle: that structure and dynamics in neural computation are not independent design choices but deeply coupled aspects of a unified system. Just as biological brains achieve their extraordinary computational efficiency through the coordinated interplay of synaptic plasticity and oscillatory synchronization, artificial architectures that respect this coupling may achieve qualitatively new capabilities. We believe that the theoretical foundations laid in this paper provide a rigorous starting point for exploring this rich design space, and we look forward to empirical validation of the predictions derived herein.

\newpage
\appendix
\onecolumngrid
\noindent\rule{\textwidth}{0.5pt}
\twocolumngrid

\section{Notation Reference}
\label{app:notation}

Table~\ref{tab:notation} summarizes the principal symbols used throughout this paper.

\begin{center}
\inlinetablecaption{\label{tab:notation}Summary of notation.}
\begin{tabular}{@{}ll@{}}
\toprule
\textbf{Symbol} & \textbf{Description} \\
\midrule
\multicolumn{2}{@{}l}{\textit{Spaces and geometry}} \\
$\Bball$ & Poincar\'{e} ball $\{\mathbf{x} \in \R^d : \norm{\mathbf{x}} < 1\}$ \\
$\dH(\mathbf{x}, \mathbf{y})$ & Hyperbolic geodesic distance \\
$\lambda_{\mathbf{x}}$ & Conformal factor $2/(1 - \norm{\mathbf{x}}^2)$ \\
$\exp_{\mathbf{z}}(\cdot)$ & Exponential map on $\Bball$ at point $\mathbf{z}$ \\
$d$ & Hyperbolic embedding dimension \\
\midrule
\multicolumn{2}{@{}l}{\textit{Oscillator dynamics}} \\
$\theta_i(t)$ & Phase of oscillator $i$ at time $t$ \\
$\omega_i$ & Intrinsic frequency of oscillator $i$ \\
$K$ & Global coupling strength \\
$C(\omega_i, \omega_j)$ & Compatibility kernel (Gaussian) \\
$\sigma_C$ & Bandwidth of compatibility kernel \\
$r(t)$ & Global macroscopic order parameter \\
$r_{\mathcal{N}_i}$ & Local order parameter over neighborhood $\mathcal{N}_i$ \\
$\psi(t)$ & Mean phase of the oscillator ensemble \\
$A_{ij}^{\text{local}}$ & Local synchronization-based attention weight \\
$T_{\text{sync}}$ & Number of phase integration steps \\
\midrule
\multicolumn{2}{@{}l}{\textit{Structural plasticity}} \\
$W_{ij}$ & Synaptic weight between units $i$ and $j$ \\
$\mathbf{W}$ & Weight matrix; $\norm{\mathbf{W}}_F$ its Frobenius norm \\
$\eta$ & Hebbian learning rate \\
$\gamma$ & Weight decay coefficient \\
$r_c$ & Critical synchronization threshold \\
$\Ind[\cdot]$ & Indicator function ($\mathbb{1}$) \\
$G(r)$ & Smooth sigmoid gate $\sigma(\beta(r - r_c))$ \\
$\beta$ & Sharpness of the sigmoid gate \\
\midrule
\multicolumn{2}{@{}l}{\textit{Architecture and learning}} \\
$\mathbf{z}_i \in \Bball$ & Hyperbolic embedding of unit $i$ \\
$\mathcal{N}_i$ & Sparse hyperbolic neighborhood of unit $i$ \\
$k$ & Maximum neighborhood size $\max_i |\mathcal{N}_i|$ \\
$\delta$ & Distance threshold for graph construction \\
$\mathbf{x}_i$ & Activation of unit $i$; $M$ its bound \\
$\sigma_a$ & Pointwise nonlinearity (e.g., ReLU) \\
$\phi(\cdot)$ & Learned projection $\R^{d_{\text{in}}} \to \Bball$ \\
$\alpha_f, \alpha_s$ & Fast and slow learning rates \\
$\mathcal{L}$ & Task loss function \\
\midrule
\multicolumn{2}{@{}l}{\textit{Analysis}} \\
$V(\mathbf{W}, \bm{\theta})$ & Composite Lyapunov function \\
$V_\theta$ & Oscillatory energy component \\
$V_W$ & Structural regularization component \\
$\lambda$ & Regularization parameter in $V$ \\
$\tau_{\text{fast}}, \tau_{\text{slow}}$ & Fast and slow characteristic timescales \\
$L_C$ & Lipschitz constant of $C(\omega_i, \omega_j)$ \\
$N$ & Number of computational units \\
\bottomrule
\end{tabular}
\end{center}

\section{Numerical Simulation Details}
\label{app:simulations}

This appendix provides the complete mathematical specifications and parameter settings for the numerical simulations presented in Figures~\ref{fig:timescale}--\ref{fig:convergence}. All simulations use forward Euler integration of the coupled Kuramoto--Hebbian ODE system.

\subsection{Figure~\ref{fig:timescale}: Two-Timescale Dynamics}

The simulation evolves $N = 50$ coupled oscillators with the following dynamics.

\textbf{Phase update (fast).} At each discrete time step $\Delta t = 0.05$:
\begin{equation}
\label{eq:app_phase}
\theta_i^{(t+1)} = \theta_i^{(t)} + \left[\omega_i + \frac{K}{N} \sum_{j=1}^{N} \sin(\theta_j^{(t)} - \theta_i^{(t)})\right] \Delta t,
\end{equation}
with coupling strength $K = 2.0$ and uniform compatibility $C(\omega_i, \omega_j) = 1$ (the classical Kuramoto limit). Intrinsic frequencies are drawn from $\omega_i \sim \mathcal{N}(0, 1)$ and centered so that $\sum_i \omega_i = 0$.

\textbf{Weight update (slow).} The smooth sigmoid gate $G(r) = \sigma(\beta(r - r_c))$ with $\beta = 20$ and $r_c = 0.5$ modulates the Hebbian update:
\begin{equation}
\label{eq:app_weight}
W_{ij}^{(t+1)} = W_{ij}^{(t)} - \gamma W_{ij}^{(t)} + \eta \, x_i^{(t)} x_j^{(t)} \cdot G\!\big(r(t)\big),
\end{equation}
with Hebbian rate $\eta = 0.01$ and weight decay $\gamma = 0.001$. Activations $x_i$ are generated from sparse random patterns (30\% sparsity, entries $\sim \mathcal{N}(0, 0.25)$). The weight matrix is initialized as a small symmetric random matrix ($W_{ij} \sim \mathcal{N}(0, 10^{-4})$, symmetrized, zero diagonal). The simulation runs for 1000 steps with random seed 42.

\subsection{Figure~\ref{fig:coupling}: Synchronization--Plasticity Coupling}

A smaller system of $N = 8$ oscillators is simulated for 2000 steps ($\Delta t = 0.02$, $K = 3.0$) to demonstrate the coupling mechanism. The intrinsic frequencies are drawn from two clusters:
\begin{equation}
\omega_i \sim \begin{cases}
\mathcal{N}(0, 0.09) & i = 1, \ldots, 5 \;\text{(tight cluster)}, \\
\mathcal{N}(3.0, 0.09) & i = 6, 7, 8 \;\text{(offset cluster)},
\end{cases}
\end{equation}
then centered. The gating and weight dynamics follow Eq.~\eqref{eq:app_weight} with $\eta = 0.02$, $\gamma = 0.002$, $\beta = 20$, and activations derived from the oscillator phases as $x_i = \text{clip}(\cos(\theta_i)/2 + 0.5 + \epsilon_i, \;0, 1)$ where $\epsilon_i \sim \mathcal{N}(0, 0.0025)$. The weight matrix is initialized at zero. Synchronized and desynchronized groups are identified post hoc via hierarchical clustering on circular phase distances. Random seed: 123.

\subsection{Figure~\ref{fig:convergence}: Lyapunov Convergence Basin}

The 3D Lyapunov surface is computed on a projected two-dimensional coordinate system $(w, \bar{\phi})$, where $w = \norm{\mathbf{W}}_F / N$ is the normalized weight Frobenius norm and $\bar{\phi}$ is the mean phase deviation from the consensus phase $\psi$. In this projection, the composite Lyapunov function (Eq.~\ref{eq:lyapunov}) reduces under a mean-field approximation to:
\begin{equation}
\label{eq:app_lyapunov}
V(w, \bar{\phi}) = -\frac{K}{2}\cos^2(\bar{\phi}) + \frac{\lambda}{2} w^2,
\end{equation}
where the $\cos^2(\bar{\phi})$ term arises from approximating $\frac{1}{N^2}\sum_{i,j} \cos(\theta_i - \theta_j) \approx \cos^2(\bar{\phi})$ for phases clustered around the mean with deviation $\bar{\phi}$. The parameters are $K = 2.0$ and $\lambda = 0.3$. A small Gaussian perturbation is added to the rendered surface for visual realism; it does not affect the theoretical analysis or trajectory dynamics.

Trajectories are computed by numerically integrating the full coupled Kuramoto--Hebbian system with $N = 20$ oscillators for 600 steps ($\Delta t = 0.02$), using the smooth sigmoid gate ($\beta = 15$, $r_c = 0.5$, $\eta = 0.01$, $\gamma = 0.001$). Eight trajectories are initialized from random phase configurations $\theta_i \sim \text{Uniform}[0, 2\pi)$ and random symmetric weight matrices ($W_{ij} \sim \mathcal{N}(0, 0.01)$), then projected onto the $(w, \bar{\phi}, V)$ space at each step. Intrinsic frequencies are drawn from $\mathcal{N}(0, 0.25)$ and centered. Random seed: 42.

\bibliography{references}

\end{document}